\documentclass[11pt,a4paper]{article}

\usepackage[utf8]{inputenc}
\usepackage[T1]{fontenc}
\usepackage{amsmath,amssymb,amsfonts}
\usepackage{graphicx}
\usepackage{hyperref}
\usepackage{natbib}
\usepackage{booktabs}
\usepackage{xcolor}
\usepackage{authblk}
\usepackage{etoolbox}

\usepackage{algorithm}
\usepackage{algpseudocode}

\pretocmd{\thebibliography}{\setlength{\itemsep}{0pt}}{}{}
\makeatletter

\usepackage{amsthm}
\newtheorem{Theorem}{Theorem}
\newtheorem{Lemma}{Lemma}

\newcommand{\doi@aux}[1]{\appto\@currentHref{#1}}
\makeatother

\usepackage[]{todonotes}

\title{Expected Free Energy-based Planning \\ as Variational Inference}

\author[1,2]{Bert de Vries\thanks{bert.de.vries@tue.nl}}
\author[1,2]{Wouter Nuijten}
\author[1]{Thijs van de Laar}
\author[1]{Wouter Kouw}
\author[1]{Sepideh Adamiat}
\author[1]{Tim Nisslbeck}
\author[1]{Mykola Lukashchuk}
\author[1]{Hoang Minh Huu Nguyen}
\author[1]{Marco Hidalgo Araya}
\author[1]{Rapha\"{e}l Tr\'{e}sor}
\author[1]{Thijs Jenneskens}
\author[1]{Ivana Nikoloska}
\author[1,3]{Raaja Ganapathy Subramanian}
\author[1,2]{Bart van Erp}
\author[1,2]{Dmitry Bagaev}
\author[1,2]{Albert Podusenko}

\affil[1]{Eindhoven University of Technology, Eindhoven, the Netherlands}
\affil[2]{Lazy Dynamics B.V., Eindhoven, the Netherlands}
\affil[3]{ASML, Veldhoven, the Netherlands}

\begin{document}

\maketitle

\begin{abstract}
We address the problem of planning under uncertainty, where an agent must choose actions that not only achieve desired outcomes but also reduce uncertainty. Traditional methods often treat exploration and exploitation as separate objectives, lacking a unified inferential foundation. Active inference, grounded in the Free Energy Principle, provides such a foundation by minimizing Expected Free Energy (EFE), a cost function that combines utility with epistemic drives, such as ambiguity resolution and novelty seeking. However, the computational burden of EFE minimization had remained a significant obstacle to its scalability. In this paper, we show that EFE-based planning arises naturally from minimizing a variational free energy functional on a generative model augmented with preference and epistemic priors. This result reinforces theoretical consistency with the Free Energy Principle by casting planning under uncertainty itself as a form of variational inference. Our formulation yields policies that jointly support goal achievement and information gain, while incorporating a complexity term that accounts for bounded computational resources. This unifying framework connects and extends existing methods, enabling scalable, resource-aware implementations of active inference agents.
\end{abstract}

\textbf{Keywords:} Active Inference, 
Bounded Rationality, Epistemic Uncertainty, Expected Free Energy, Free Energy Principle, Planning as Inference, Policy Optimization, Variational Inference

\tableofcontents

\section{Introduction}

Planning under uncertainty is a fundamental challenge in both artificial intelligence and cognitive neuroscience. Agents must select actions that not only achieve desired outcomes but also reduce uncertainty about their environment. Classical approaches—rooted in reinforcement learning and optimal control—typically address this by optimizing long-term utility through value function estimation or policy learning \cite{sutton2018reinforcement, bertsekas2012dynamic}. However, these methods often treat reward maximization (exploitation) and uncertainty reduction (exploration) as separate objectives, using heuristics to strike a balance between them. Moreover, they struggle in high-dimensional or deep-horizon settings due to compounding errors and the curse of dimensionality.

Active inference offers a principled alternative. Grounded in the Free Energy Principle (FEP)\footnote{We use the following abbreviations in this paper: Expected Free Energy (EFE), Free Energy Principle (FEP),  Kullback-Leibler (KL),  Planning as Inference (PAI), Variational Free Energy (VFE). }, it casts perception, learning, and action selection as inference processes that minimize a variational bound on surprise \cite{friston2010free, parr2022active}. Central to this framework is the Expected Free Energy (EFE), a unified objective that combines instrumental (goal-directed) and epistemic (information-seeking) components \cite{friston2015active}. Minimizing EFE yields behavior that simultaneously pursues preferred outcomes and resolves uncertainty, providing a theoretically grounded solution to the exploration–exploitation trade-off.

Despite its promise, practical implementations of EFE-based planning remain computationally demanding \cite{kappen2012optimal, palmieri_unifying_2022,van_de_laar_realizing_2024, friston_sophisticated_2021, Paul-Predictive-planning-2024}. Existing methods often resort to approximations that undermine alignment with the FEP, particularly the foundational claim that all processing arises from variational free energy minimization. In this work, we address these limitations.

We show that EFE-based planning can be rigorously formulated as variational inference on a generative model augmented with preference and epistemic priors. Our central result demonstrates that minimizing a well-defined variational free energy functional naturally yields policies that integrate goal-directed behavior, information-seeking exploration, and bounded rationality. This formulation improves full theoretical alignment with the FEP, and unifies active inference with the broader planning-as-inference paradigm, offering a scalable and principled framework for decision-making under uncertainty.

The next section introduces the formal definition of EFE and highlights its desirable properties for planning under uncertainty. Section~\ref{sec:related-work} reviews prior work on EFE minimization and outlines several limitations of existing approaches. The central contribution of this paper—demonstrating how EFE minimization can be recast as standard variational inference—is presented as a formal theorem in Section~\ref{sec:EFE-theorem}. The paper concludes with a discussion of the theorem’s implications and its relevance for building scalable active inference agents.

\section{The Expected Free Energy Cost Function}\label{sec:EFE-cost-function}

Consider an agent described by a generative model $p(yx\theta u)$.\footnote{For brevity, we omit commas in the notation for joint variables, e.g., $p(yx\theta u) = p(y, x, \theta, u) $.}  In this paper, we are only concerned with planning, so we will assume that the model predicts a sequence of future observations. A typical example of this model would be a rollout of a state space model, for instance
\begin{align}\label{eq:generative-model}
    p(yx\theta u) = p(x_t) p(\theta)\underbrace{\prod_{k=t+1}^T p(y_k|x_k,\theta) p(x_k|x_{k-1},u_k) p(u_k)}_{\text{rollout to the future}}
\end{align}

where $t$ holds the current time step. In this model, $y$ denotes the sequence of future observations, $x$ represents the (latent) states, $\theta$ contains the model parameters, and $u$ refers to the policy, i.e., a sequence of future actions (controls). Because all of these variables are defined as part of a model rollout into the future, they are all treated as unobserved variables. Since \eqref{eq:generative-model} is designed to predict how the future is expected to unfold, we refer to it as the predictive model. 

In model \eqref{eq:generative-model}, the prior distribution $p(u)$ can be understood as an empirical distribution over allowable policies based on contextual data. Assume that we are additionally provided with a distribution 
\begin{equation}
 \hat{p}(x)\,,   
\end{equation}
which describes the \emph{preferred} future states, sometimes referred to as target states. The \emph{planning} objective is to infer a policy posterior $q(u)$ that, if executed, would efficiently guide the agent to these preferred states.\footnote{Since $q(u)$ is conditioned solely on priors, namely the generative model $p(yx\theta u)$ and a preference prior $\hat{p}(x)$, it would be more appropriate to speak about an updated prior $q(u)$ rather than a posterior. For simplicity, in this paper, all distributions that result from inference are denoted by $q(\cdot)$ and termed posteriors.} 

In the active inference literature, candidate policies are evaluated by a cost function $G(u)$, known as the Expected Free Energy, which is defined as
\begin{equation}\label{eq:G=r+a-n}
   G(u) = \underbrace{E_{q}\bigg[\log \frac{q(x|u)}{\hat{p}(x)}\bigg]}_{\text{risk}} +   \underbrace{\underbrace{E_{q}\bigg[\log\frac{1}{q(y|x  )} \bigg]}_{\text{ambiguity}}  -  \underbrace{E_{q}\bigg[\log \frac{ q(\theta|y
   x)}{ q(\theta|x )} }_{\text{novelty}}\bigg]}_{\text{epistemic costs}}\,,
\end{equation}
where the expectations are with respect to $q = q(yx\theta|u)$.
In \eqref{eq:G=r+a-n}, the assumption is that $q(yx\theta u)$ is the variational posterior with respect to the joint model $p(yx\theta u)\hat{p}(x)$. In section~\ref{sec:EFE-theorem}, we will refine this statement. Policies with a lower $G(u)$ value are regarded as more favorable, i.e., a priori more likely to be selected. Active inference processes are based on the Free Energy Principle, a sophisticated theory grounded in core physics concepts that accounts for the behavior of living systems as if $G(u)$ were their planning cost function. The most contemporary reference is \citet{friston_path_2023}. On a more practical level, we summarize this theory by discussing the incentives behind the three components of $G(u)$. 
\begin{itemize}
    \item \textbf{Risk} refers to the KL-divergence between $q(x|u)$, which is the state that we expect to reach under policy $u$, and the target state $\hat{p}(x)$. EFE minimization aligns with risk minimization.   
    \item \textbf{Ambiguity} is the expected entropy $E_{q(x|u)}\big[H[q(y|x)] \big]$ of future observations $y$ under policy $u$. Minimizing EFE results in policies that seek well-predicted (i.e., unambiguous) observations, leading to accurate state estimates. 
    \item \textbf{Novelty} extends information-seeking policies to include \emph{active} parameter learning. Minimizing EFE results in policies that maximize the mutual information between observations $y$ and parameters $\theta$.   
\end{itemize}

EFE minimization can be viewed as a unifying framework for planning under uncertainty, integrating principles from both decision theory and optimal control. Several established paradigms emerge as special cases of EFE minimization under specific assumptions. For example, Kullback-Leibler (KL) control \cite{todorov2006linearly, rawlik2013stochastic} arises when epistemic (information-seeking) terms are omitted, effectively reducing EFE to a risk-based utility optimization. Conversely, when the risk term is removed, EFE minimization reduces to Bayesian experimental design \cite{lindley1956bayesian}, which focuses purely on maximizing information gain.

As an aside, if preferences were instead described by a distribution $\hat{p}(y)$ over desired future observations, it is common to define an alternative EFE as

\begin{equation}\label{eq:G=p-s-n}
   G^\prime(u) = \underbrace{E_{q}\bigg[\log \frac{1}{\hat{p}(y)}\bigg]}_{\substack{\text{pragmatic}\\ \text{costs}}}    \underbrace{-\underbrace{E_{q}\bigg[\log\frac{q(x|y)}{q(x|u  )} \bigg]}_{\text{salience}}  -  \underbrace{E_{q}\bigg[\log \frac{ q(\theta|yx)}{ q(\theta|x )} }_{\text{novelty}}\bigg]}_{\text{epistemic costs}}\,,
\end{equation}
where, as before, the expectations are with respect to $q=q(yx\theta|u)$. Although $G(u)$ and $G^\prime(u)$ are distinct cost functions with different definitions of epistemic costs, they are centered around the same criteria. They can be regarded as specific instances of a broader policy cost function template given by 
\begin{equation}
 E_{q(yx\theta|u)}\bigg[\log \frac{q(x\theta|u)}{\bar{p}(yx\theta)}\bigg]
\end{equation}
where the denominator features a ``biased'' model $\bar{p}(yx\theta)$, incorporating both preference incentives $\hat{p}(\cdot)$ and variational posteriors $q(\cdot)$. Specifically, $G(u)$ in \eqref{eq:G=r+a-n} uses $\bar{p}(yx\theta) = q(\theta|yx) q(y|x) \hat{p}(x)$, and $\bar{p}(yx\theta)$ in \eqref{eq:G=p-s-n} factorizes as $\bar{p}(yx\theta) = q(\theta|yx) q(x|y) \hat{p}(y)$. In the following discussion, we focus on $G(u)$, noting that similar derivations and arguments apply to $G^\prime(u)$.

\section{Related Work}\label{sec:related-work}
We shortly review recent efforts on how to find policies that minimize EFE efficiently. 

In the Sophisticated Inference (SI) framework introduced by \cite{friston_sophisticated_2021}, a layer of recursive belief modeling is added to the EFE formulation. This enables deeper forms of planning, where agents consider not only “What will happen if I do this?” but also “What will I believe will happen if I do this?”—allowing for richer introspective evaluation of future outcomes.

While conceptually compelling, the computational implementation of SI relies on an explicit tree search over candidate policies. As the planning horizon increases, the number of possible policy sequences grows combinatorially, making the exhaustive evaluation of EFE increasingly intractable.

To address this scalability issue, Paul et al. proposed the Dynamic Programming Expected Free Energy (DPEFE) framework \cite{Paul-Predictive-planning-2024}. By leveraging dynamic programming principles, DPEFE computes expected free energy recursively, reducing the computational cost of long-horizon planning. This reformulation allows active inference agents to plan efficiently in more complex environments without sacrificing theoretical rigor.

A conceptual limitation of approaches such as Sophisticated Inference and dynamic programming-based policy selection is that they rely on explicitly designed, human-crafted algorithms for policy selection. This sits uncomfortably with the FEP, which posits that all cognitive and behavioral processes should emerge from the automatic, event-driven minimization of variational free energy. From this perspective, policy selection should ideally arise entirely through an inference process, rather than through externally imposed algorithmic procedures.

The motivation for adopting a Planning-as-Inference (PAI) perspective is not merely philosophical alignment with the FEP; it also stems from practical considerations. Specifically, policy selection should be interruptible—capable of producing a valid approximate result at any time—and should scale gracefully with available computational resources. These properties are naturally afforded by embedding policy selection within a reactive message-passing scheme on a factor graph, where each local message incrementally reduces free energy \citep{bagaev_reactive_2023}. In such a framework, computation is inherently flexible and distributed, and intermediate solutions retain semantic coherence. In contrast, algorithmic approaches based on procedural code—with nested loops and conditionals—lack this interruptibility and adaptability, making them ill-suited for real-time or resource-constrained settings.

The PAI framework, proposed initially by \cite{attias2003planning} and later extended by \cite{toussaint2009robot} and \cite{solway2012optimal}, reinterprets planning as a probabilistic inference problem: the goal is to infer action trajectories that are most consistent with prior preferences over outcomes. This perspective enables the use of approximate inference techniques, such as variational inference and message passing, to develop computationally efficient planning algorithms.

However, the mentioned PAI formulations focus on maximizing expected utility and do not explicitly incorporate epistemic value, i.e., the drive to reduce uncertainty, which is a defining feature of EFE–based approaches. As a result, their applicability in highly uncertain or partially observable environments is limited, as they lack a principled mechanism for information-seeking behavior.

\citet{palmieri_unifying_2022} introduced a comprehensive framework that unifies estimation and control through belief propagation on factor graphs, with a particular emphasis on path planning applications. Building on this perspective, \citet{ van_de_laar_realizing_2024} extended the PAI framework by integrating epistemic value into the policy evaluation process, enabling agents to account for both expected utility and information gain during planning. Specifically, they propose modifying the Variational Free Energy (VFE) by subtracting a mutual information term when inference is performed over future (i.e., planned) segments of the factor graph. This adjustment allows reactive message passing to naturally account for both instrumental and epistemic value, yielding an interruptible and entirely local inference procedure for evaluating candidate policies.

In contrast to above mentioned PAI methods, the approach proposed by Van de Laar and Koudahl yields results that align with EFE minimization, but unfortunately it also introduces some conceptual and practical challenges. Conceptually, it is somewhat inelegant to alternate between different cost functions depending on the location of computation within the factor graph. This bifurcation undermines the principle of a unified objective function underlying all natural inference processes. Practically, it complicates the design and implementation of inference toolkits: developers must now account for two distinct message computations for each node—one for standard inference and another for planning—thereby significantly increasing implementation complexity and reducing modularity.

Finally, outside the FEP community and more within the reinforcement learning literature, the recent work by \citet{lazaro2024planning} offers a compelling perspective on the relationship between planning and inference. Similar to our approach, their work highlights the role of entropy and information-seeking behavior in planning. The key distinction lies in the formulation of the inference objective: while \citet{lazaro2024planning} demonstrate that planning corresponds to a specific weighting of entropy terms within a general variational objective, we introduce a VFE functional for a generative model that is augmented with epistemic priors, which yields EFE-based planning as a natural consequence.

In the next section, we develop a PAI framework that is not only consistent with the FEP but also addresses some of the conceptual and practical limitations of the approaches discussed above.

\section{EFE-based Planning as Variational Inference}\label{sec:EFE-theorem}
The main contribution of the paper is described by a theorem, which we conveniently label as the Expected Free Energy theorem. 

\begin{Theorem}[Expected Free Energy Theorem]\label{the:EFE}
Consider an agent with generative (predictive) model
\begin{align}
    p(yx\theta u)\,,   
\end{align}
and prior beliefs 
\begin{equation}
\hat{p}(x)    
\end{equation}
about future desired states. 

Let the Variational Free Energy functional $F[q]$ be defined as
\begin{equation}\label{eq:VFE-for-planning}
F[q] \triangleq E_{q(yx\theta u)}\bigg[ \log \frac{ \overbrace{q(yx\theta u)}^{\text{posterior}} }{ \underbrace{p(yx\theta u)}_{\substack{\text{generative} \\ \text{model}} } \underbrace{\hat{p}(x)}_{\substack{\text{preference}\\ \text{prior}}} \underbrace{\tilde{p}(u) \tilde{p}(x) \tilde{p}(yx)}_{\text{epistemic priors}}} \bigg]\,,
\end{equation}
where the generative model in the denominator is augmented by both a preference prior $\hat{p}(\cdot)$ and epistemic priors $\tilde{p}(\cdot)$.

Let the epistemic priors be defined as
\begin{subequations}\label{eq:epistemic-priors}
\begin{align}
   \tilde{p}(u) &= \exp(H[q(x|u)]) \label{eq:epistemic-prior-u}\\
   \tilde{p}(x) &= \exp(-H[q(y|x)]) \label{eq:epistemic-prior-x}\\
   \tilde{p}(yx) &= \exp(D[ q(\theta|yx), q(\theta|x)])\,.\label{eq:epistemic-prior-xy}
\end{align}    
\end{subequations}

Then,
$F[q]$ decomposes as
\begin{align}\label{eq:F=G+complexity} 
F[q] = \underbrace{E_{q(u)}\big[ G(u)\big]}_{\substack{ \text{expected policy} \\ \text{costs} }}  + \underbrace{E_{q(yx\theta u)}\bigg[\log \frac{q(yx\theta u)}{p(yx\theta u)}\bigg]}_{\text{complexity}} \,.
\end{align}
where $G(u)$ is the expected free energy as defined in \eqref{eq:G=r+a-n}. 
In \eqref{eq:epistemic-priors}, $H[q] = E_q[-\log q]$ is the entropy functional, and $D[q,p] = E_q[\log q - \log p]$ is the Kullback-Leibler divergence (see Appendix~\ref{sec:entropy-kullback}).
\end{Theorem}
\begin{proof}
    The proof of \eqref{eq:F=G+complexity} is given in Appendix~\ref{sec:proof-main-theorem}.
\end{proof}

A key consequence of \eqref{eq:F=G+complexity} is that minimizing $F[q]$ leads to reducing the expected policy costs $E_{q(u)}[G(u)]$, while also balancing this with the drive to reduce complexity, which are the costs associated with changing beliefs.

As shown in Appendix~\ref{sec:proof-main-theorem}, if one wishes to enforce normalization of the prior distributions, the epistemic prior $\tilde{p}(u)$ in \eqref{eq:epistemic-prior-u} can alternatively be chosen as
\begin{equation}
\tilde{p}(u) =
\sigma(H[q(x|u)]) \triangleq \frac{\exp(H[q(x|u)])}{\sum_{u'} \exp(H[q(x|u')])},
\end{equation}
where $\sigma$ denotes the softmax function applied to the entropy $H[q(x|u)]$. Similar constructions apply to the priors $\tilde{p}(x)$ and $\tilde{p}(xy)$. In that case, the variational free energy $F[q]$ is shifted by a constant that does not affect the location of its minimum, and thus does not influence the outcome of the optimization process.

\section{Discussion}

\subsection{Optimal Planning by Variational Inference}
Starting from \eqref{eq:F-C-2} (see proof in Appendix~\ref{sec:proof-main-theorem}), we can compute the optimal policy posterior through 
\begin{align}
  F[q] &= E_{q(u)}\bigg[ \log \frac{q(u)}{p(u)} 
    + G(u) +\underbrace{E_{q(yx\theta | u)} \big[\log \frac{q(yx\theta|u)}{p(yx\theta|u)}\big]}_{=C(u) \text{ (complexity)}}\bigg] \notag \\
    &= E_{q(u)}\bigg[ \log \frac{q(u)}{\exp\big(-P(u) -G(u) - C(u)\big)}
     \bigg] \,, \label{eq:F-with-G-and-B}
\end{align}
where $P(u) = -\log p(u)$ denotes the policy prior expressed as a cost function. Equation \eqref{eq:F-with-G-and-B} is a Kullback-Leibler divergence that is minimized for
\begin{align}
      q^*(u) &= \arg\min_q F[q] \notag \\
      &= \sigma\big(-P(u) - G(u) - C(u) \big)\,, \label{eq:q*-PGB}
\end{align}
where 
\begin{equation}\label{eq:sigmoid}
    \sigma(a)_k = \frac{\exp(a_k)}{\sum_{k'} \exp(a_{k'})}
\end{equation} is a normalized exponential function.  

Equation \eqref{eq:q*-PGB} is not new. A comparable formula for the optimal policy can be found in Equation 2.1 of \cite{friston_sophisticated_2021}. A main contribution of this paper is to demonstrate that the (previously established) optimal policy, given by \eqref{eq:q*-PGB}, can be obtained through standard variational minimization of an appropriately defined free energy functional $F[q]$.

\subsection{Interpretation of the Epistemic Priors}

The epistemic prior $\tilde{p}(u) = \exp(H[q(x|u)])$, introduced in \eqref{eq:epistemic-prior-u}, imposes a bias toward selecting policies that maximize the entropy over future states $x$. This reflects an information-seeking preference, as high entropy states indicate that the agent is maintaining flexibility—keeping future states open for adaptation. Additionally, the epistemic prior $\tilde{p}(x) = \exp(-H[q(y|x)])$ in \eqref{eq:epistemic-prior-x}, favors policies that reduce the uncertainty over future states by selecting observations that are informative about them. Together, $\tilde{p}(u)$ and $\tilde{p}(x)$ induce a bias toward ambiguity-minimizing behavior. Similarly, the priors $\tilde{p}(u)$ and $\tilde{p}(y, x)$ jointly shape a preference for policies that maximize novelty.

\subsection{On the Complexity Term \texorpdfstring{$C(u)$}{C(u)}}

In \eqref{eq:q*-PGB}, $P(u)$ and $G(u)$ reflect past and future information about the policy posterior $q(u)$, respectively. The complexity term $C(u)$ represents the discrepancy (expressed as a KL divergence) between the (inferred) variational posterior $q(yx\theta|u)$ and the (ideal) Bayesian posterior $p(yx\theta|u)$.

Is $C(u)$ simply an unavoidable cost lacking benefits since, in contrast to $P(u)$ and $G(u)$, it does not provide information regarding effective policies? Not quite. Inference of $q$ must be executed on a specific platform in a particular context that provides access to a certain set of computational resources. For example, the resources available for tracking a specific car in traffic may differ depending on the overall complexity of the traffic situation. The presence of bounded computational resources can be regarded as a constraint on the inference process. 

Typical inference constraints include mean-field assumptions on $q$, as well as assumptions about the posterior form (e.g., $q(u)$ must be Gaussian, even if $p(u)$ is not). Latency assumptions also apply; for instance, $q(u)$ may need to be executed within 5 milliseconds, regardless of the state of the inference process. 

While these inference constraints are not incorporated into the objective function $F[q]$ as seen in \eqref{eq:VFE-for-planning}, the term $C(u)$ can be understood as a drive to minimize the unavoidable effects of these inference constraints. In other words, due to the complexity term $C(u)$ in \eqref{eq:q*-PGB}, active inference agents that minimize \eqref{eq:VFE-for-planning} are Bayes-optimal planners for a given set of constraints. We refer the reader to \cite{senoz_variational_2021}for a discussion on how to express inference constraints explicitly in the objective function.

\subsection{PAI in a Synthetic Active Inference Agent}\label{sec:implementation}

We discuss here how the inference process within an active inference agent may proceed. Consider the generative model for a dynamical system given by
\begin{equation}
\prod_{k=1}^T p(y_k|x_k) p(x_k|x_{k-1},u_k)\,
\end{equation}
which is represented by the white nodes in the factor graph shown in Fig.\ref{fig:generative-model-with-preferences}. We assume that both the initial and desired final states of this system are constrained by priors $\hat{p}(x_0|x^+)$ and $\hat{p}(x_T|x^+)$, respectively. These priors are generated by a higher-level state $x^+$, and are visualized as orange (initial) and blue (target) nodes in Fig.\ref{fig:generative-model-with-preferences}.

At time step $k = 0$, the agent’s task is to infer a sequence of future actions $u_{1:T}$ such that the expected posterior $q(x_T|y_{1:T})$ over the final state closely matches the target distribution $\hat{p}(x_T|x^+)$. Inference proceeds entirely via spontaneous ("reactive") message passing in the factor graph, without external orchestration.

Fig.~\ref{fig:generative-model-with-preferences} illustrates the state of this process at time step $t$, following the execution of actions $u_{1:t}$ and the observation of outcomes $y_{1:t}$. The green and red terminal nodes representing future actions and future states correspond to epistemic priors. At time $t$, the system’s future rollout comprises the generative model terminated by both epistemic and target priors.

As time progresses, inference within this future model remains active: green and red epistemic priors are progressively replaced by posterior factors (depicted as small black boxes), each replacement introducing new opportunities for free energy minimization through message passing. Consequently, beliefs over the remaining policy $u_{t+1:T}$ continue to evolve as the system integrates new information.

This inference mechanism can, in principle, be executed entirely automatically using a reactive message passing toolbox such as \texttt{RxInfer}~\citep{bagaev_rxinfer_2023}. We intend to describe simulations of this process in forthcoming publications.

\begin{figure}[tb!]
\centering
\includegraphics[width=\textwidth]{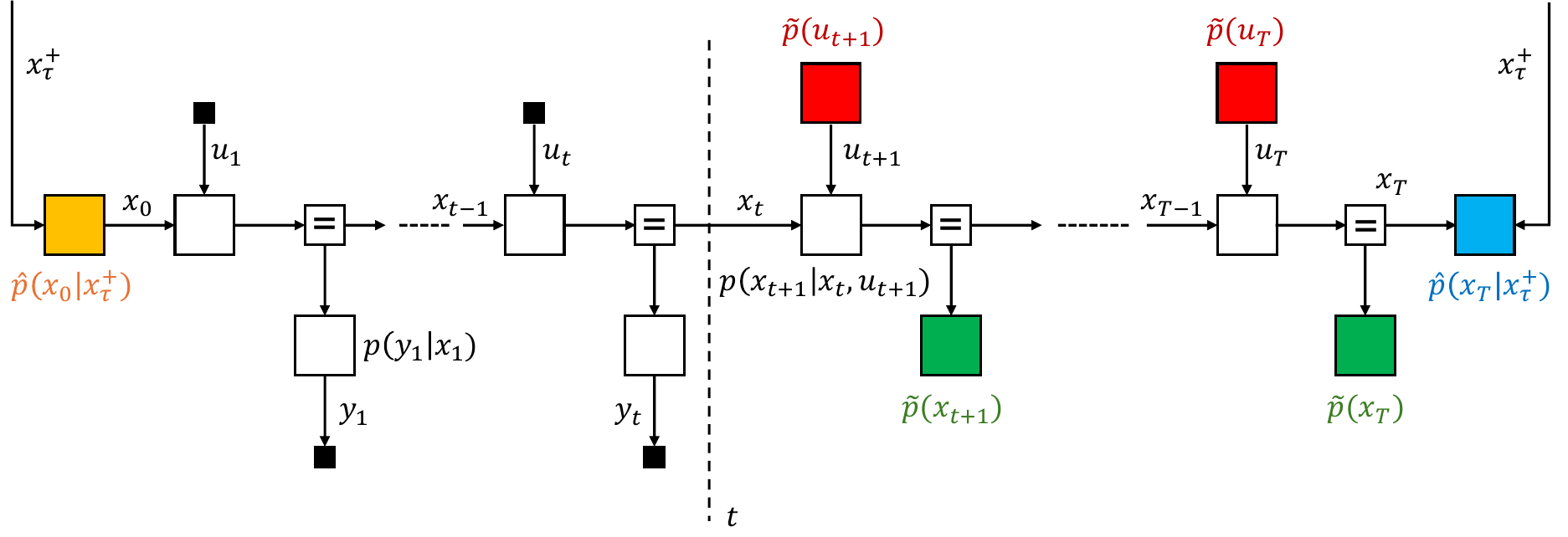}
\caption{The state of a (layer in an) active inference agent during the inference process. See section~\ref{sec:implementation} for details.}
\label{fig:generative-model-with-preferences}
\end{figure}

\subsection{Toward Scalable Synthetic Active Inference}\label{sec:scalable-AIF}

The results presented in this paper may also pave the way for scalable and energy-efficient active inference agents. In particular, we consider how the main theorem could be extended to support message computation at the level of individual nodes within a factor graph. This discussion outlines a promising direction for future research.

In~\cite{vries_toward_2023}, we proposed that the implementation of synthetic active inference agents should follow the procedure outlined in Algorithm~\ref{alg:reactive}. In essence, the inference process should not rely on any hand-crafted algorithms beyond the instruction to react whenever an opportunity arises to minimize (variational) free energy—provided the agent’s energy budget permits. This process is to be implemented via reactive message passing in a factor graph, enabling a fully autonomous and distributed inference mechanism.

\begin{figure}[!htb]
\begin{algorithm}[H]
\centering
\caption{Idealized Implementation of an Active Inference Agent}\label{alg:reactive}
\begin{algorithmic}[1]
\State \textbf{Specify initial model} $p(y,x,\theta,u)$ and \textbf{priors} $\hat{p}(x)$
\While {\textrm{true}} \Comment{Deployment loop}
\State React to any free energy minimization opportunity
\EndWhile
\end{algorithmic}
\end{algorithm}
\caption{Pseudo-code for realizing a synthetic active inference agent. See section~\ref{sec:scalable-AIF} for details.}
\label{fig:AIF-algorithms}
\end{figure}

A key phrase here is “react to any opportunity,” which underscores that the agent (or at a finer level of abstraction, any node in the factor graph) should
only compute anything when there is an actionable opportunity to minimize free energy. If we interpret the decision to compute (and send) a variational message versus not computing (and remaining silent) as an action choice,
then the EFE of these alternatives can serve as a decision criterion. A message should only be computed and sent if the EFE associated with doing so is lower than the EFE of abstaining. The EFE Theorem opens the door to evaluating the EFE of action choices via standard variational free energy
minimization within an appropriately extended generative model. Thus, we foresee an energy-efficient and fully autonomous active inference process, driven solely by localized variational free energy minimization. In this framework, hand-crafted planning algorithms based on ad hoc pruning in a tree search process are replaced by an autonomously operating Bayes-optimal reactive inference process. A toolbox like RxInfer would, in principle, be
capable of automating this process \cite{bagaev_rxinfer_2023}.

\subsection{Limitations}

This work also raises several open questions. While, in theory, it is possible to rank policy alternatives by EFE through VFE minimization on a generative model equipped with customized epistemic priors, the current results remain conceptual: implementation details are lacking, and validation through simulations has yet to be performed. To simplify the mathematical exposition, we omitted time indices from the variables; however, incorporating time explicitly, as would be required in a full dynamical system specification, may introduce additional complexity that warrants further elaboration.

Moreover, although the epistemic priors in \eqref{eq:epistemic-priors} are given in closed-form expressions, their practical implementation and online updating procedures are not straightforward. In a factor graph framework, for example, message passing through nodes representing these epistemic priors would likely require pre-computation or approximation strategies.

In summary, we view the contributions presented in this paper as a conceptual foundation for a line of research aimed at realizing PAI within active inference agents.

\section{Conclusions}

We have presented a principled formulation of planning under uncertainty by casting Expected Free Energy minimization as a problem of variational inference. Our central result shows that EFE-based policy optimization naturally emerges from minimizing a variational free energy functional defined over a generative model augmented with preference and epistemic priors. This formulation restores theoretical alignment with the Free Energy Principle, resolving previous challenges where planning and inference were treated as conceptually distinct operations.

By treating all inference, including policy selection, as message passing in a factor graph, our framework supports scalable, interruptible, and fully distributed planning. This perspective not only strengthens the theoretical foundation of active inference but also opens the door to practical implementations using reactive message-passing toolkits. These results pave the way for the design of synthetic active inference agents that are fully self-organizing and capable of performing Bayes-optimal planning without relying on handcrafted algorithms.

\bibliographystyle{plainnat}
\bibliography{bert-library}

\begin{thebibliography}{22}
\providecommand{\natexlab}[1]{#1}
\providecommand{\url}[1]{\texttt{#1}}
\expandafter\ifx\csname urlstyle\endcsname\relax
  \providecommand{\doi}[1]{doi: #1}\else
  \providecommand{\doi}{doi: \begingroup \urlstyle{rm}\Url}\fi

\bibitem[Attias(2003)]{attias2003planning}
Hagai Attias.
\newblock Planning by probabilistic inference.
\newblock In \emph{Advances in Neural Information Processing Systems}, volume~16, 2003.

\bibitem[Bagaev and de~Vries(2023)]{bagaev_reactive_2023}
Dmitry Bagaev and Bert de~Vries.
\newblock Reactive {Message} {Passing} for {Scalable} {Bayesian} {Inference}.
\newblock \emph{Scientific Programming}, 2023:\penalty0 e6601690, May 2023.
\newblock ISSN 1058-9244.
\newblock \doi{10.1155/2023/6601690}.
\newblock URL \url{https://www.hindawi.com/journals/sp/2023/6601690/}.
\newblock Publisher: Hindawi.

\bibitem[Bagaev et~al.(2023)Bagaev, Podusenko, and De~Vries]{bagaev_rxinfer_2023}
Dmitry Bagaev, Albert Podusenko, and Bert De~Vries.
\newblock {RxInfer}: {A} {Julia} package for reactive real-time {Bayesian} inference.
\newblock \emph{Journal of Open Source Software}, 8\penalty0 (84):\penalty0 5161, April 2023.
\newblock ISSN 2475-9066.
\newblock \doi{10.21105/joss.05161}.
\newblock URL \url{https://joss.theoj.org/papers/10.21105/joss.05161}.

\bibitem[Bertsekas(2012)]{bertsekas2012dynamic}
Dimitri Bertsekas.
\newblock \emph{Dynamic Programming and Optimal Control, volume 2, 4th edition}.
\newblock Athena Scientific, 2012.

\bibitem[De~Vries(2023)]{vries_toward_2023}
Bert De~Vries.
\newblock Toward {Design} of {Synthetic} {Active} {Inference} {Agents} by {Mere} {Mortals}.
\newblock \emph{CoRR}, abs/2307.14145, 2023.
\newblock \doi{10.48550/ARXIV.2307.14145}.
\newblock URL \url{https://doi.org/10.48550/arXiv.2307.14145}.
\newblock arXiv: 2307.14145.

\bibitem[Friston(2010)]{friston2010free}
Karl Friston.
\newblock The free-energy principle: a unified brain theory?
\newblock \emph{Nature Reviews Neuroscience}, 11\penalty0 (2):\penalty0 127--138, 2010.

\bibitem[Friston et~al.(2015)Friston, FitzGerald, Rigoli, Schwartenbeck, and Pezzulo]{friston2015active}
Karl Friston, Thomas FitzGerald, Francesco Rigoli, Philipp Schwartenbeck, and Giovanni Pezzulo.
\newblock Active inference and epistemic value.
\newblock \emph{Cognitive Neuroscience}, 6\penalty0 (4):\penalty0 187--214, 2015.

\bibitem[Friston et~al.(2021)Friston, Da~Costa, Hafner, Hesp, and Parr]{friston_sophisticated_2021}
Karl Friston, Lancelot Da~Costa, Danijar Hafner, Casper Hesp, and Thomas Parr.
\newblock Sophisticated {Inference}.
\newblock \emph{Neural Computation}, 33\penalty0 (3):\penalty0 713--763, March 2021.
\newblock ISSN 0899-7667.
\newblock \doi{10.1162/neco_a_01351}.
\newblock URL \url{https://doi.org/10.1162/neco_a_01351}.

\bibitem[Friston et~al.(2023)Friston, Da~Costa, Sakthivadivel, Heins, Pavliotis, Ramstead, and Parr]{friston_path_2023}
Karl Friston, Lancelot Da~Costa, Dalton A.~R. Sakthivadivel, Conor Heins, Grigorios~A. Pavliotis, Maxwell Ramstead, and Thomas Parr.
\newblock Path integrals, particular kinds, and strange things.
\newblock \emph{Physics of Life Reviews}, 47:\penalty0 35--62, December 2023.
\newblock ISSN 1571-0645.
\newblock \doi{10.1016/j.plrev.2023.08.016}.
\newblock URL \url{https://www.sciencedirect.com/science/article/pii/S1571064523001094}.

\bibitem[Kappen et~al.(2012)Kappen, G{\'o}mez, and Opper]{kappen2012optimal}
H.J. Kappen, V.~G{\'o}mez, and M.~Opper.
\newblock Optimal control as a graphical model inference problem.
\newblock \emph{Machine Learning}, 87\penalty0 (2):\penalty0 159--182, 2012.
\newblock \doi{10.1007/s10994-011-5252-8}.

\bibitem[Lindley(1956)]{lindley1956bayesian}
Dennis~V. Lindley.
\newblock Bayesian statistics and the design of experiments.
\newblock \emph{The Annals of Mathematical Statistics}, 27\penalty0 (2):\penalty0 568--578, 1956.
\newblock \doi{10.1214/aoms/1177728069}.
\newblock URL \url{https://projecteuclid.org/euclid.aoms/1177728069}.

\bibitem[Lázaro-Gredilla et~al.(2024)Lázaro-Gredilla, Ku, Murphy, and George]{lazaro2024planning}
Miguel Lázaro-Gredilla, Li~Yang Ku, Kevin~P. Murphy, and Dileep George.
\newblock What type of inference is planning?
\newblock In \emph{Advances in Neural Information Processing Systems}, 2024.
\newblock URL \url{https://proceedings.neurips.cc/paper_files/paper/2024/hash/d39e3ae9a11b79691709a7a6e06a63d9-Abstract-Conference.html}.

\bibitem[Palmieri et~al.(2022)Palmieri, Pattipati, Gennaro, Fioretti, Verolla, and Buonanno]{palmieri_unifying_2022}
Francesco A.~N. Palmieri, Krishna~R. Pattipati, Giovanni~Di Gennaro, Giovanni Fioretti, Francesco Verolla, and Amedeo Buonanno.
\newblock A {Unifying} {View} of {Estimation} and {Control} {Using} {Belief} {Propagation} {With} {Application} to {Path} {Planning}.
\newblock \emph{IEEE Access}, 10:\penalty0 15193--15216, 2022.
\newblock ISSN 2169-3536.
\newblock \doi{10.1109/ACCESS.2022.3148127}.

\bibitem[Parr et~al.(2022)Parr, Pezzulo, and Friston]{parr2022active}
Thomas Parr, Giovanni Pezzulo, and Karl Friston.
\newblock \emph{Active Inference: The Free Energy Principle in Mind, Brain, and Behavior}.
\newblock MIT Press, 2022.

\bibitem[Paul et~al.(2024)Paul, Isomura, and Razi]{Paul-Predictive-planning-2024}
Aswin Paul, Takuya Isomura, and Adeel Razi.
\newblock On predictive planning and counterfactual learning in active inference.
\newblock \emph{Entropy}, 26\penalty0 (6), 2024.
\newblock ISSN 1099-4300.
\newblock \doi{10.3390/e26060484}.
\newblock URL \url{https://www.mdpi.com/1099-4300/26/6/484}.

\bibitem[Rawlik et~al.(2013)Rawlik, Toussaint, and Vijayakumar]{rawlik2013stochastic}
Konrad Rawlik, Marc Toussaint, and Sethu Vijayakumar.
\newblock Stochastic optimal control as approximate inference: A new perspective.
\newblock \emph{Proceedings of the International Conference on Machine Learning (ICML)}, 2013.

\bibitem[Solway and Botvinick(2012)]{solway2012optimal}
Alec Solway and Matthew~M Botvinick.
\newblock Optimal behavioral hierarchy.
\newblock \emph{PLoS Computational Biology}, 8\penalty0 (10):\penalty0 e1002774, 2012.

\bibitem[Sutton and Barto(2018)]{sutton2018reinforcement}
Richard~S Sutton and Andrew~G Barto.
\newblock \emph{Reinforcement Learning: An Introduction, 2nd edition}.
\newblock MIT Press, Cambridge, MA, 2018.

\bibitem[Todorov(2006)]{todorov2006linearly}
Emanuel Todorov.
\newblock Linearly-solvable markov decision problems.
\newblock \emph{Advances in neural information processing systems}, 19:\penalty0 1369--1376, 2006.

\bibitem[Toussaint(2009)]{toussaint2009robot}
Marc Toussaint.
\newblock Robot trajectory optimization using approximate inference.
\newblock In \emph{Proceedings of the 26th Annual International Conference on Machine Learning}, pages 1049--1056. ACM, 2009.

\bibitem[van~de Laar et~al.(2024)van~de Laar, Koudahl, and de~Vries]{van_de_laar_realizing_2024}
Thijs van~de Laar, Magnus Koudahl, and Bert de~Vries.
\newblock Realizing {Synthetic} {Active} {Inference} {Agents}, {Part} {II}: {Variational} {Message} {Updates}.
\newblock \emph{Neural Computation}, pages 1--38, September 2024.
\newblock ISSN 0899-7667.
\newblock \doi{10.1162/neco_a_01713}.
\newblock URL \url{https://doi.org/10.1162/neco_a_01713}.

\bibitem[Şenöz et~al.(2021)Şenöz, van~de Laar, Bagaev, and de~Vries]{senoz_variational_2021}
İsmail Şenöz, Thijs van~de Laar, Dmitry Bagaev, and Bert de~Vries.
\newblock Variational {Message} {Passing} and {Local} {Constraint} {Manipulation} in {Factor} {Graphs}.
\newblock \emph{Entropy (Basel, Switzerland)}, 23\penalty0 (7):\penalty0 807, June 2021.
\newblock ISSN 1099-4300.
\newblock \doi{10.3390/e23070807}.

\end{thebibliography}


\appendix

\section{Proof of the Main Theorem}\label{sec:proof-main-theorem}

\begin{proof}[Proof of Theorem~\ref{the:EFE}]
\begin{subequations}
   \begin{align}
    F[q] &= E_{q(y x \theta u )}\bigg[ \log \frac{q(y x \theta u )}{p(y x \theta u)  \hat{p}(x) \tilde{p}(u) \tilde{p}(x)  \tilde{p}(yx)} \bigg] \\
    &= E_{q(u)}\bigg[ \log \frac{q(u)}{p(u)} 
    + \underbrace{E_{q(yx\theta | u)}\big[ \log \frac{q(y x \theta | u)}{p(yx \theta|u)  \hat{p}(x) \tilde{p}(u) \tilde{p}(x)  \tilde{p}(yx)}\big]}_{B(u)}  
     \bigg] \label{eq:F-C-1}\\
     &= E_{q(u)}\bigg[ \log \frac{q(u)}{p(u)} 
    + \underbrace{G(u) +E_{q(yx\theta | u)} \big[\log \frac{q(yx\theta|u)}{p(yx\theta|u)}\big]}_{B(u) \text{ if \eqref{eq:epistemic-priors} holds}}  
     \bigg] \label{eq:F-C-2} \\
    &= E_{q(u)}\big[ G(u)\big]+ E_{q(yx\theta u)}\bigg[\log \frac{q(yx\theta u)}{p(yx\theta u)}\bigg] \quad \text{if \eqref{eq:epistemic-priors} holds }  
\end{align} 
\end{subequations}
\end{proof}

In the above derivation, we still need to prove the transition for $B(u) $ from 
\eqref{eq:F-C-1} to \eqref{eq:F-C-2}, which we address next. 
In the following, all expectations are with respect to $q(yx\theta|u)$ unless otherwise indicated. 

\begin{Lemma}[Proof of equivalence $B(u)$ in \eqref{eq:F-C-1} and \eqref{eq:F-C-2}]
\begin{subequations}\label{eq:proof-Cu}
\begin{align}
B(&u) = E\bigg[ \log \frac{ \overbrace{q(yx\theta|u)}^{\text{posterior}} }{ \underbrace{p(yx\theta|u)}_{\text{predictive}} \underbrace{\hat{p}(x)}_{\text{utility}} \underbrace{\tilde{p}(u) \tilde{p}(x) \tilde{p}(yx)}_{\text{epistemic priors}}} \bigg] \label{eq:Cu-first-line} \\
&= \underbrace{ E\bigg[\log\bigg( \underbrace{\frac{q(x|u)}{\hat{p}(x)}}_{\text{risk}}\cdot \underbrace{\frac{1}{q(y|x  )}}_{\text{ambiguity}} \cdot \underbrace{\frac{ q(\theta|x)}{ q(\theta|yx )}}_{-\text{novelty}} \bigg) \bigg] }_{G(u) = \text{Expected Free Energy}} +   \label{eq:Cu-eqC} \\
&\quad + E\bigg[ \log\bigg( \underbrace{\frac{\hat{p}(x) q(y|x ) q(\theta| yx)}{q(x|u) q(\theta|x)}}_{\text{inverse factors from }G(u)} \cdot \underbrace{\frac{q(yx\theta|u)}{p(yx\theta|u) \hat{p}(x) \tilde{p}(u) \tilde{p}(x) \tilde{p}(yx) }}_{\text{factors from }\eqref{eq:Cu-first-line}} \bigg)\bigg] \notag \\
&= G(u) + \underbrace{E\bigg[ \log \frac{q(yx\theta|u)}{p(yx\theta|u)}\bigg]}_{=C(u)} + \underbrace{E\bigg[ \log  \frac{q(y|x ) q(\theta|yx)}{q(x|u) q(\theta|x) \tilde{p}(u) \tilde{p}(x) \tilde{p}(yx)} \bigg]}_{\text{choose epistemic priors to let this vanish}} \\
&= G(u) + C(u) +  \\
&\quad + E\bigg[\log \frac{1}{q(x|u) \tilde{p}(u)} \bigg] + E\bigg[ \log  \frac{q(y|x)}{\tilde{p}(x)} \bigg] + E\bigg[ \log  \frac{q(\theta|yx)}{q(\theta|x) \tilde{p}(yx) } \bigg] \notag \\
&= G(u) + C(u) +  \\
&\qquad + \sum_{y\theta} q(y\theta|x) \bigg( \underbrace{\underbrace{-\sum_x q(x|u) \log q(x|u)}_{= H[q(x|u)]} - \sum_x q(x|u) \log \tilde{p}(u)}_{=0 \text{ if }\tilde{p}(u) = \exp(H[q(x|u)])}\bigg)  \label{eq:tilde-p-vanish} \\
&\qquad + \sum_{x} q(x|u) \bigg( \underbrace{\underbrace{\sum_{y} q(y|x) \log q(y|x)}_{= -H[q(y|x)]} - \sum_{y} q(y|x) \log \tilde{p}(x)}_{=0 \text{ if }\tilde{p}(x) = \exp(-H[q(y|x)])} \bigg)   \notag \\
&\qquad + \sum_{yx} q(yx|u) \bigg( \underbrace{\underbrace{\sum_\theta q(\theta|yx) \log \frac{q(\theta|yx)}{q(\theta|x)}}_{D[q(\theta|yx),q(\theta|x)]} - \sum_\theta q(\theta|yx) \log \tilde{p}(yx)}_{=0 \text{ if } \tilde{p}(yx) = \exp(D[q(\theta|yx),q(\theta|x)])} \bigg) \notag \\
&= G(u) + E_{q(yx\theta|u)}\bigg[ \log \frac{q(yx\theta|u)}{p(yx\theta|u)}\bigg] \quad \text{if \eqref{eq:epistemic-priors} holds.}
\end{align}
\end{subequations}
\end{Lemma}




The term in brackets in \ref{eq:tilde-p-vanish} works out as follows:
\begin{subequations}
    \begin{align}
   -\sum_x & q(x|u) \log q(x|u) - \sum_x q(x|u) \log \tilde{p}(u) \\
    &= H[q(x|u)] - \log \tilde{p}(u) \\
    &= \begin{cases} 0  & \text{if }\tilde{p}(u) = \exp(H[q(x|u)]) \\
    \text{const} & \text{if }\tilde{p}(u) = \sigma(H[q(x|u)])\end{cases}
\end{align}
\end{subequations}
where $\text{const} = \log\left( \Sigma_{u'} \exp(H\left[q(x|u')\right]\right)$ and
\begin{equation}
\sigma(H[q(x|u)]) \triangleq \frac{\exp(H\left[q(x|u)\right])}{\Sigma_{u'} \exp(H\left[q(x|u')\right]} 
\end{equation}
is the (normalized) softmax function. Similar derivations apply to the other epistemic priors in \eqref{eq:epistemic-priors}. 

In the context of variational inference with variational free energy $F[q]$ as in \eqref{eq:VFE-for-planning}, the normalization of $\tilde{p}(u)$ is inconsequential, as the additive constant does not affect the location of the minimum of $F[q]$. As long as $\tilde{p}(u) \propto \exp(H[q(x|u)])$, the results of VFE minimization will be the same.

\section{The Entropy and Kullback-Leibler Divergence}\label{sec:entropy-kullback}
In \eqref{eq:epistemic-prior-u}, the \emph{entropy} of the conditional distribution $q(x|u)$ is defined as 
    \begin{equation}\label{eq:cond-entropy-prime}
        H[q(x|u)] = - \sum_x q(x|u) \log q(x|u)\,.
    \end{equation}
Note that $H[q(x|u)]$ is a function of $u$ and therefore $\sigma(H[q(x|u)])$ can serve as a probability distribution over $u$. $H[q(x|u)]$ is not the same as the \emph{conditional entropy} $H^\prime[q(x|u)]$, which is a scalar, defined as
     \begin{equation}
\underbrace{H^\prime[q(x|u)]}_{\substack{\text{conditional}\\ \text{entropy}}} \triangleq - \sum_{xu} q(xu) \log q(x|u) = E_{q(u)}\left[H[q(x|u)]\right]\,. \label{eq:cond-entropy}
    \end{equation}  

For two given distributions $q(\theta|yx)$ and $q(\theta|x)$, the Kullback-Leibler divergence is defined as
    \begin{align}
         D[ q(\theta|yx), q(\theta|x) ] &\triangleq \sum_\theta q(\theta|yx) \log \frac{q(\theta|yx)}{q(\theta|x)} \notag \\
         &= \sum_\theta q(\theta|yx) \log \frac{q(y\theta|x)}{q(y|x)q(\theta|x)} \label{eq:I-prime}\,,
    \end{align}
    which is a function of $y$ and $x$. Note that \eqref{eq:I-prime} is not the same as, but is close to the \emph{mutual information} between $y$ and $\theta$, given $x$, which is a scalar value defined as
    \begin{align}
         I[y,\theta|x] &\triangleq \sum_{yx\theta} q(yx\theta) \log \frac{q(y\theta|x)}{q(y|x)q(\theta|x)} \notag \\
         &= \sum_{yx} q(yx) D[ q(\theta|yx), q(\theta|x) ] \notag\\
         &= E_{q(yx)}\left[ D[ q(\theta|yx), q(\theta|x) ]\right]\,.  \label{eq:mut-inf}
    \end{align}    
    Note the similarity between \eqref{eq:cond-entropy} and \eqref{eq:mut-inf}.
\end{document}